\documentclass{article} 
\pdfoutput=1
\usepackage{iclr2022_conference,times}

\usepackage{amsmath,amsfonts,bm}

\def\eqref#1{equation~\ref{#1}}

\def\1{\bm{1}}

\DeclareMathAlphabet{\mathsfit}{\encodingdefault}{\sfdefault}{m}{sl}
\SetMathAlphabet{\mathsfit}{bold}{\encodingdefault}{\sfdefault}{bx}{n}

\usepackage{amsthm,amsmath,amssymb}
\usepackage{mathrsfs}
\usepackage{hyperref}
\usepackage{url}
\usepackage[utf8]{inputenc} 
\usepackage[T1]{fontenc}    
\usepackage{booktabs}       
\usepackage{amsfonts}       
\usepackage{nicefrac}       
\usepackage{microtype}      
\usepackage{xcolor}         
\usepackage{multicol}
\usepackage{multirow}
\usepackage{graphicx}
\usepackage{threeparttable}
\usepackage{wrapfig}

\newtheorem{definition}{Definition}
\newtheorem{theorem}{Theorem}

\newtheorem{lemma}{Lemma}[section]
\newtheorem{assumption}{Assumption}

\title{Density-Based Clustering with Kernel Diffusion}

\iclrfinalcopy

\author{
Zheng Chao$^{~\ast}$ \\
School of Mathematics Sciences \\
University of Southampton \\
Southampton, England \\
\texttt{chao.zheng@southampton.ac.uk} \\
\And
Yingjie Chen \thanks{Equal contribution.} \\
School of EECS \\
Peking University \\
Beijing, China \\
\texttt{chenyingjie@pku.edu.cn} \\
\And
Chong Chen \thanks{Corresponding author.} \\
DAMO Academy, Alibaba Group \\
Beijing, China \\
\texttt{cheung.cc@alibaba-inc.com} \\
\And
Jianqiang Huang \\
DAMO Academy, Alibaba Group \\
Beijing, China \\
\texttt{jianqiang.jqh@gmail.com} \\
\And
Xian-Sheng Hua \\
DAMO Academy, Alibaba Group \\
Beijing, China \\
\texttt{huaxiansheng@gmail.com} \\
}

\begin{document}

\maketitle

\begin{abstract}
Finding a suitable density function is essential for density-based clustering algorithms such as DBSCAN and DPC.
A naive density corresponding to the indicator function of a unit $d$-dimensional Euclidean ball is commonly used in these algorithms. Such density  suffers from capturing local features in complex datasets. To tackle this issue, we propose a new kernel diffusion density function, which is adaptive to data of varying local distributional characteristics and smoothness. Furthermore, we develop a surrogate that can be efficiently computed in linear time and space and prove that it is asymptotically equivalent to the kernel diffusion density function. Extensive empirical experiments on benchmark and large-scale face image datasets show that the proposed approach not only achieves a significant improvement over classic density-based clustering algorithms but also outperforms the state-of-the-art face clustering methods by a large margin.

\end{abstract}

\section{Introduction}
\label{sec:intro}

Density-based clustering algorithms are now widely used in a variety of applications, ranging from high energy physics~\citep{Tramacere2012, Rovere2020}, material sciences~\citep{Marquis2019, Reza2007}, social network analysis~\citep{shi2014, Khatoon2019} to molecular biology~\citep{cao2017, Ziegler2020}. In these algorithms, data points are partitioned into clusters that are considered to be sufficiently or locally high-density areas with respect to an underlying probability density or a similar reference function. We call them density functions throughout this paper. These techniques are attractive to practitioners, due to their non-parametric feature, which leads to flexibility in discovering clusters that have arbitrary shapes, whilst classic methods such as $k$-means and $k$-medoids~\citep{ESL} can only detect convex (e.g., spherical) clusters. 
Seminal work in the context of density-based clustering includes DBSCAN~\citep{dbscan} and DPC~\citep{dpc}, among many others~\citep{Ankerst1999, Cuevas2001, Comaniciu2002, Hinneburg2007, Stuetzle2003}. 

 Most density-based clustering algorithms implicitly identify cluster centers and assign remaining points to the clusters by connecting with the higher density point nearby. To proceed with these methods it requires a density function, which is usually an estimate of the underlying true probability  density or some variants of it. For example, a popular choice is the naive density function that is carried out by simply calculating the number of data points covered in the $\varepsilon$-neighborhood of each $x$. Note that such densities are not adaptive to different distribution regions. One of the  challenging scenarios is when clusters in the data have varying local features, for example, size, height, spread, and smoothness. Therefore, the resulting density function has a tendency to flatten the peaks and valleys in the data distribution, which leads to underestimation of the number of clusters (see Figure 1). Many heuristics variations of DBSCAN and DPC have been proposed to magnify the local features, thus making the clustering task easier \citep{HDBSCAN, LC, SNN, den-ratio}. Most of these methods can be viewed as performing clustering on certain transformations of the naive density function. However, if the naive density function itself is quite problematic in the first place, these methods will become less effective.
\begin{figure}[t!]
    \caption{(a) Data generated from Gaussian mixture model with 3 components, each has differing variance and weight. (b) Naive density function in 2D (top) and 3D (bottom): only one peak can be identified. (c) Proposed kernel diffusion density function: 3 clusters can be easily discovered.}
    \centering
    \includegraphics[width=0.9\textwidth]{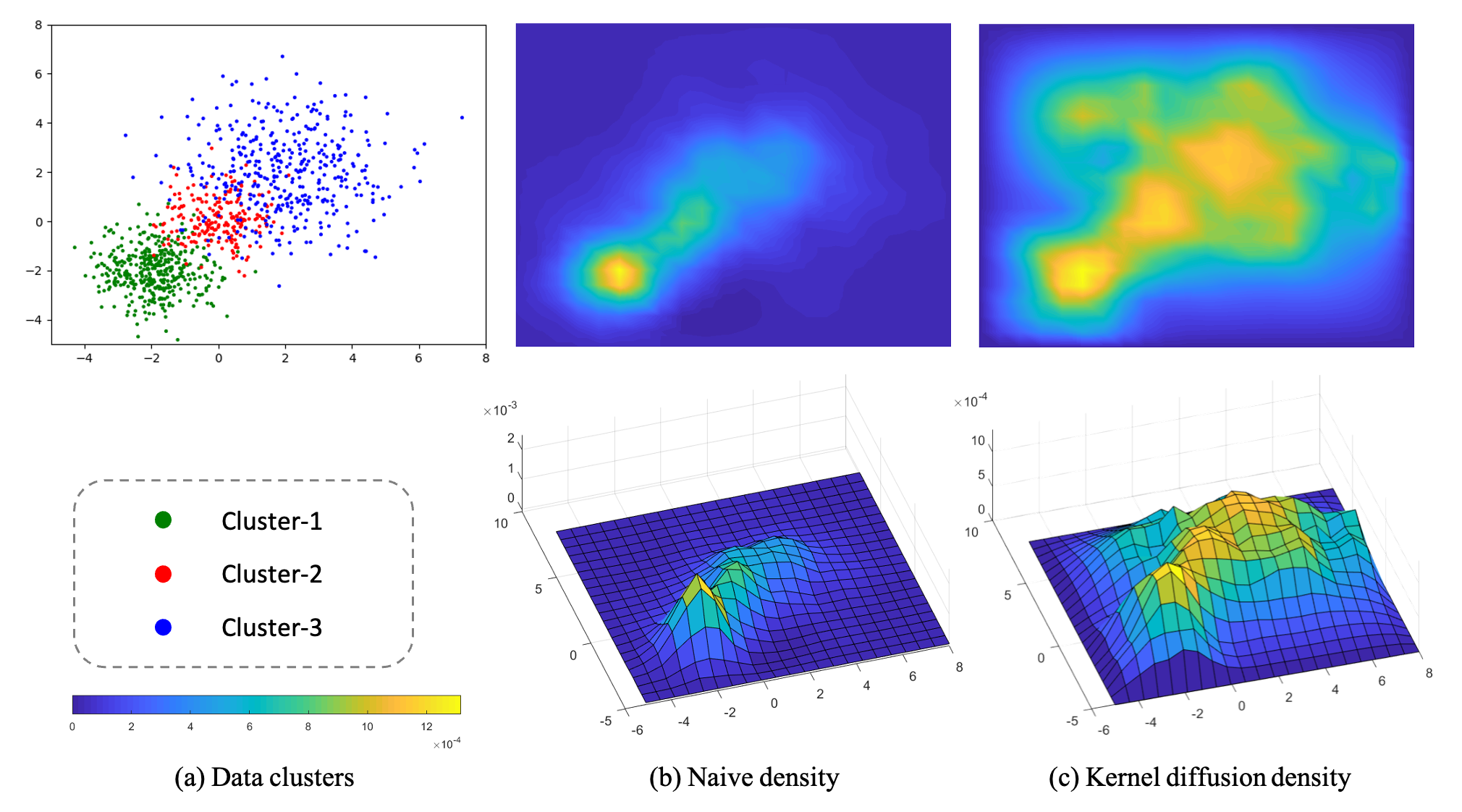}
    \label{fig:motivation}
\end{figure}

Moreover, even if we 
apply adaptive alternatives to modify the classic density functions, there are other contentious issues of the generally used linear kernel density estimator (KDE). It often suffers from severe boundary bias \citep{Marron1996} and is acknowledged as computationally expensive. 
These phenomenon prevent the classic density functions being practically useful and reliable,  
especially for large-scale and complex clustering tasks.



To overcome these problems in density-based clustering algorithms, in this paper we propose a general approach to build the so-called kernel diffusion density function to replace classic density functions. The key idea is to construct the density from a user-specified bivariate kernel  that has desired local adaptive properties. Instead of using the naive density function and its variants, we utilize the bivariate kernel to derive a transition probability. A diffusion process is induced by this transition probability,  which admits a limiting and stationary distribution. This limiting distribution serves as a plausible  
density function for clustering with reduced error.

Under this framework, we provide examples of symmetric and asymmetric bivariate kernels to construct the kernel diffusion density function, which can tackle clustering complex and locally varying data. We apply the resulting adapted DBSCAN and DPC algorithms to widely different empirical datasets and show significant improvement in each of these analyses. The main contributions of this paper are summarized below:
\begin{itemize}
    \item We introduce new bivariate kernel functions and construct the associated kernel diffusion processes. Based on the diffusion process, we propose a kernel diffusion density function to adapt density-based clustering algorithms such as DBSACN and DPC, which attains accuracy in the presence of varying local features.
    \item We derive a computationally much more efficient surrogate, and show analytically it is asymptotic equivalent to the proposed kernel diffusion density function.
    \item By extensive experiments, we demonstrate the superiority of kernel diffusion density function over naive density function and its variants when applying to DBSCAN and DPC, and show it outperforms state-of-the-art GCN-based methods on face clustering tasks.
\end{itemize}

\section{Related Work}

\paragraph{Density-Based Clustering} There is vast literature on adapting density-based clustering algorithms to tackle large variations in different clusters in the data.
DPC itself is such a refinement of DBSCAN, as it determines cluster centers not only by highest density values but also by taking into account their distances from each other, thus has a generally better performance in complex clustering tasks. Other attempts include rescaling the data to have relative reference measures instead of KDE~\citep{den-ratio, LC}, and using the number of shared-nearest-neighbors between two points to replace the geometric distance~\citep{SNN}.


\paragraph{Diffusion Maps}
The technique of diffusion maps~\citep{Coifman2005,Coifman2006} gives a multi-scale organization of the data according to their underlying geometric structure. 
It uses a local similarity measure to create a diffusion process on the data which integrates local geometry at different scales along time $t$. Generally speaking, the diffusion will segment the data into several smaller clusters in small $t$ and group data into one cluster for large $t$. Applying eigenfunctions at a carefully selected time $t$ leads to good macroscopic representations of the data, which is useful in dimension reduction and spectral clustering~\citep{Nadler2005}. 

\paragraph{Face Clustering}
Face clustering has been extensively studied as an important application in machine learning. Traditional algorithms include $k$-means, hierarchical clustering~\citep{hac} and ARO~\citep{aro}. Many recent works take advantage of supervised information and GCN models, achieving impressive improvement comparing to traditional algorithms. To name a few, CDP~\citep{cdp} proposes to aggregate the features extracted by different models; L-GCN~\citep{lgcn} predicts the linkage in an instance pivot subgraph; LTC~\citep{ltc} generates a series of subgraphs as proposals and detects face clusters thereon; and GCN(V+E)~\citep{gcn-ve} learns both the confidence and connectivity by GCN. In this paper we demonstrate that the proposed density-based clustering algorithm with kernel diffusion, as a general clustering approach, even outperforms theses state-of-the-art methods that are especially designed for face clustering.

\section{Preliminaries}
\label{sec:preliminaries}
\subsection{Notations}

Let the dataset $D=\{x_1,\dots, x_n\} \subset \mathbb{R}^d$ be $n$ i.i.d samples drawn from a distribution measure $F$ with density $f$ on $\mathbb{R}^d$. Let $F_n$ denote the corresponding empirical distribution measured with respect to $D$, i.e., $F_n(A) = \frac{1}{n}\sum_{i=1}^n \textbf{1}_{A}(x_i)$, where $\textbf{1}_A(\cdot)$ denotes the indicator function of set $A$. We write $||u||$ as the Euclidean norm of vector $u$.
Let $B(x, \varepsilon)$ and $V_d$ denote the $d$-dimensional $\varepsilon$-ball centered at $x$ and the volume of the unit ball $B(0, 1)$, respectively.  
Let $N_k(x)$ denote the set of $k$-nearest neighbors of point $x$ within the dataset $D$.

\subsection{Density function}
Density-based algorithms perform  clustering by specifying and segmenting high-value areas in a density function denoted by $\rho$. Usually, we calculate each of $\rho(x_i)$, and then  identify cluster centers with (locally) highest values. Many popular algorithms such as DBSCAN and DPC employ the following naive density function:
\begin{equation}
\label{equ:density-naive}
\rho_{\text{naive}}(x) = \dfrac{1}{n \varepsilon^d}\sum_{y\in D} \dfrac{\textbf{1}_{B(x, \varepsilon) }(y)}{V_d}.
\end{equation}
The naive density function $\rho_{\text{naive}}$ is actually an empirical estimation of $f$ for carefully chosen $\varepsilon$. It is easy to observe, for clustering purpose we only care about $\rho_{\text{naive}}(x)$ up to a normalising constant, which makes it simply equivalent to counting the total number of data points in the $\varepsilon$-ball around $x$. 

In practice, the data distribution may be very complex and contains varying local features that are difficult to be detected. The naive density in (\ref{equ:density-naive}) with the same radius $\varepsilon$ for all $x$ usually suffers from unsatisfactory empirical performance, for example, failing to identify small clusters with fewer data points. One possible way to alleviate this problem is through a transformation into the following local contrast (LC) function \citep{LC}:
\begin{align}
\label{equ:lc}
    \rho_{\text{LC}}(x) = 
     & \dfrac{1}{n}\sum_{y\in N_{k}(x)}\textbf{1}_{\rho_{\text{naive}}(x) > \rho_{\text{naive}}(y)}.
\end{align}
In this way, $\rho_{\text{LC}}$ compares the density of each data point with its $k$-nearest neighbors. To see the benefit of LC, let us consider $x$ to be a cluster center. After local contrasting, $\rho_{\text{LC}}(x)$ is likely to reach the value of $k$ regardless of the size of this cluster. 

However density functions like $\rho_{\text{LC}}$ still highly depend on the underpinning performance of $\rho_{\text{naive}}$. This restricts their applications in clustering data with challenging local features.

\section{Methodology}

In this section, we present a new type of density-based clustering algorithm, based on the notion of kernel diffusion density function. Towards this end, we will introduce a kernel diffusion density function, which takes account of local adaptability and is well-suited for clustering purpose. We  provide details on how to derive this density function from a diffusion process induced by bivariate kernels. We also provide a surrogate density function that is computationally more efficient.


\subsection{Diffusion process and Kernel Diffusion Density}
 \label{sec:diffusion}
 
Considering a bivariate kernel function $k: D\times D \rightarrow \mathbb{R}^+$, such that:
\begin{itemize}
 \item $k(x, y)$ is positive semi-definite, i.e., $k(x, y)\ge 0$.
 \item $k(x, y)$ is $F_n$-integrable with respect to both $x$ and $y$. 
\end{itemize}
We define $d(x)=n\int_{D}k(x,y)dF_n(y)$ as a local measure of the volume at $x$ and define
\begin{equation} \label{eq:transition}
p(x,y)=\dfrac{k(x,y)}{d(x)}.
\end{equation}
It is easy to see that $p(x,y)$ satisfies the conservation property $n\int_{\mathcal{D}}p(x,y)dF_n(y)=1$.  
As a result, $p(x,y)$ can be viewed as a probability for a random walk on the dataset from point $x$ to point $y$, which induces a Markov chain on $D$ with $n\times n$ transition matrix $P=[p(x,y)]$.  This technique is standard in various applications, known as the normalized graph Laplacian construction. For example, we can view $D$ as a graph, $L=I-P$ as the normalized graph Laplacian, and $d(x)$ as a normalization factor.

For $t\ge 0$, the probability of transiting from $x$ to $y$ in $t$ time steps is given by $P^t$, the $t$-th power of $P$.  Running the Markov chain forward in time, we observe the dataset at different scales, which is the diffusion process $X_t$ on $D$. Let $\rho(x,t)$: $D\times\mathbb{R}^{+} \rightarrow \mathbb{R}^{+}$ be the associated probability density, which is governed by the following second-order differential equation with initial conditions:
\begin{equation}
\label{eq:diffusion}
      \begin{cases}
    &\frac{\partial}{\partial t} \rho(x,t) = -L \rho(x,t), \\
    & \rho(x,0) =  \phi_0(x),
     \end{cases}
\end{equation} 
where $\phi_0(x)$ is a probability density at time $t = 0$. In practice we can use any valid choice of $\phi_0(x)$, e.g. the uniform density.


To give an explicit example of the diffusion process, consider a sub-class of $k$, i.e., isotropic kernels, where  $k(x,y)=\mathcal{K}(||x-y||^2/h)$ for some function $\mathcal{K}: \mathbb{R}\rightarrow \mathbb{R}^+$. Here we can dual interpret $h$  as a scale parameter to infer local information and as a  time step $h=\Delta t$ at which the random walk jumps. Then we can define the forward Chapman-Kolmogorov operator $T_F$ as
 $$
 T_F=n\int_{D} p(x,y) \phi_0(y)dF_n(y).
 $$
Note that  $T_F$ is the data distribution at time $t = h$, thus can be viewed as continuous analogues of the left multiplication by the transition matrix $P$. Letting $h \rightarrow 0$, the random walk converges to a continuous diffusion process with probability density evolves continuously in $t$. In this case, we can explicitly write the second-order differential equation in (\ref{eq:diffusion}) as:
\begin{equation}
\label{equ:diffusion1}
    \frac{\partial}{\partial t} \rho(x,t) =\lim_{h\rightarrow 0}\dfrac{\hat \rho(x,t+h)- \rho(x,t)}{h}= \lim_{h\rightarrow 0} \dfrac{T_F-I}{h} \rho(x,t), \\
\end{equation}
where $L_h=\lim_{h\rightarrow 0} {(T_F-I)}/{h}$ is the conventional infinitesimal generator of the process.

Now we are ready to introduce our kernel diffusion density function.
\begin{definition}(Kernel diffusion density function)
Suppose the Markov chain induced by $P$ is ergodic, we define the kernel diffusion density function as the limiting probability density of the diffusion process $X_t$, i.e.,
\begin{equation}
\label{equ:density-dfiiusion} 
    \rho_{\rm{KD}}(x) = \lim_{t \to \infty}\rho(x,t).
\end{equation}
\end{definition}

We provide some intuitions that why $\rho_{\text{KD}}$ serves as a valid density function for clustering. With increased values of $t$, the diffusion process $X_t$ reveals the geometric structure (such as high-density regions) of the data distribution $F$ at increasing scales. To see this, note that the transition probability $P$ reflects connectivity between data points. We can interpret a cluster as a region in which the probability of staying in this region is high during a transition. The probability of following a path along an underlying geometric structure increases with $t$, as the involved data points are dense and highly connected. The paths, therefore, form along with short and high probability jumps. Whilst paths that do not follow this structure form long and low probability jumps, which lowers the path’s overall probability. 

Meanwhile, to prevent the diffusion process grouping of all the data into one large cluster as $t\rightarrow \infty$, we can use certain sophisticated forms of $k(x,y)$ that focus on local adaptivity. This will help slow down the diffusion and lead the process gradually towards the correct geometry structure at the right scale. Thus they achieve a balance between magnifying geometry structures and identifying the signal of local features. 


\subsection{Locally Adaptive Kernels}
\label{sec:kernels}
To address the local adaptability in kernel diffusion density function, we propose the following two bivariate kernels. Both of them are very simple variations of the most commonly used classic kernels.
\paragraph{Symmetric-Gaussian kernel:} 
\begin{equation}
\label{equ:sdde}
k(x, y) = 
\exp\bigg(-\frac{\|x-y\|^{2}}{h}\bigg) \mathbf{1}_{B(x, \varepsilon)}(y). 
\end{equation}
Here $h$ and $\epsilon$ are both hyper-parameters. We call this kernel symmetric since $k(x, y)=k(y, x)$ .

\paragraph{Asymmetric-Gaussian kernel:}
\begin{equation}
\label{equ:adde}
k(x, y) = 
\exp\bigg(-\frac{\|x-y\|^{2}}{h}\bigg) \mathbf{1}_{N_k(x)}(y). 
\end{equation}
Here $h$ and $k$ are hyper-parameters. Note that in this case $k(x,y)$ is asymmetric as $y \in N_{k}(x)$ does not imply $x \in N_{k}(y)$.

Bivariate kernels defined in (\ref{equ:sdde}) and (\ref{equ:adde}) are just combinations of classic Gaussian kernel and $\varepsilon$-neighbourhood or $k$-nearest neighbours kernels, respectively. With these simple combinations, we truncate Gaussian kernel at local areas, and the contribution of each point $y$ to the construction of the density function $\rho_{\text{KD}}(x)$ depends not only on the distance $||y-x||$ but also on the local geometry structure around $x$. Hence, the new kernels are adaptive at different $x$, which is expected to lead to better clustering performance against local features. We remark that the Asymmetric-Gaussian kernel  takes into account a varying neighborhood around each $x$, thus is more adaptive comparing to the Symmetric-Gaussian kernel. 

Although here we only provide two examples of locally adaptive kernels,  other options can be easily created in a similar spirit under this framework, e.g., changing the Gaussian kernels to other kernels or changing the $\varepsilon$-neighbourhood ($k$-nearest neighbours) kernels to other locally truncated functions. Once $k(x,y)$ is determined, we can derive the corresponding density function $\rho_{\text{KD}}$. Next, we just need to simply apply any density clustering procedure like  DPC or DBSCAN based on $\rho_{\text{KD}}$ instead of the naive density function $\rho_{\text{naive}}$. 

  In Section \ref{sec:experiment}, we assess the empirical performance of the proposed kernel diffusion density function with the above two locally adaptive kernels. They outperform existing density-based algorithms and other state-of-the-art methods.


\subsection{Fast Kernel Diffusion Density}
\label{sec:fkd}

The kernel diffusion density function $\rho_{\text{KD}}$   can be calculated as the stationary distribution of a Markov chain induced by the  transition matrix $P$. Numerically, we can solve it by iteratively right multiplying $P$ with $\rho(x,t)$ until convergence, or applying a QR decomposition on $P$.
These methods are expensive in terms of computational cost, especially when the sample size $n$ is large. 

To tackle this problem, we propose the following surrogate of $\rho_{\text{KD}}(x)$ which is computationally more efficient.  
\begin{definition}(Fast kernel diffusion density function) Let $p(y,x)$ be the transition probability from point $y$ to point $x$, as defined in equation (\ref{eq:transition}). We define the  fast kernel diffusion density function as
\begin{equation}
\label{eq:den.ref}
\rho_{\rm{FKD}}(x)=\int_D p(y, x) dF_n(y),
\end{equation}
\end{definition}
It is straightforward that $\rho_{\text{FKD}}$ can be obtained in linear time and memory space, as we only need to compute the column averages of matrix $P$.

Here we show that $\rho_{\text{FKD}}$ is not only computationally efficient but also suitable for detecting local features. This is illustrated through the following Theorem \ref{theorem:3}. Consider a special case that $k(x,y)=\mathbf{1}_{B(x, \varepsilon)}(y)$. Then it is easy to verify that
$$\rho_{\text{FKD}}(x) = \frac{1}{C_{d}}\sum_{y\in B(x, \varepsilon)}\frac{1}{\rho_{\text{naive}}(y)},$$
where $C_{d} = n\varepsilon^{d}V_{d}$ is a normalising constant. In this way, we build a connection between $\rho_{\text{FKD}}$ and the naive density function $\rho_{\text{naive}}$ in this special example.
\begin{theorem}
\label{theorem:3}
Consider the above special case that $k(x,y)=\mathbf{1}_{B(x, \varepsilon)}(y)$. In addition, assume the dataset $D = \{x_{1},...,x_{n}\}$ can be split into $m$ disjoint clusters: i.e., $X = D_{1}\bigcup...\bigcup D_{m}$ and for each $x\in D$, $B(x, \varepsilon)$ only contain data points that belong to the same cluster as $x$. Denote $\bar{\rho}_{j} = \frac{1}{|D_{j}|}\sum_{x\in D_{j}}\rho_{\text{FKD}}(x)$ as the average density in cluster $j$.  We have
$$\bar{\rho}_{1} =\dots= \bar{\rho}_{m} = 1.$$
\end{theorem}
Theorem \ref{theorem:3} demonstrates that the averaged $\rho_{\text{FKD}}$ in each cluster are  the same regardless of cluster sizes and other local features. This shows that $\rho_{\text{FKD}}$ elevates the density of small clusters, which is essential for finding the density peaks of small clusters.

Previously we claim that $\rho_{\text{FKD}}$ is a surrogate of the kernel diffusion density $\rho_{\text{KD}}$.   Next, we want to reveal the relationship between these two density functions from an asymptotic viewpoint. To proceed, we will need the following assumption.
\begin{assumption}
 \label{assp2}
There exists some positive constant $c<1$ such that $\rho_{\text{FKD}}(x)\le c$ uniformly holds for every $x\in D$.
 \end{assumption}
This is a very mild assumption, since it always holds that $\rho_{\text{FKD}}(x)<1$, and the average of $\rho_{\text{FKD}}(x)$ over the dataset is $\int_D \rho_{\text{FKD}}(x) dF_n(x)=1/n$, which vanishes as $n\rightarrow\infty$. Now we are ready to present the following theorem that characterise the uniform closeness between $\rho_{\text{FKD}}$ and $\rho_{\text{KD}}$. 

\begin{theorem}
\label{theorem:2}
Suppose that Assumption \ref{assp2} holds and the  Markov chain induced by the kernel $k(x,y)$ is ergodic. We have 
$$\lim_{n\to \infty}\left|\frac{\rho_{\text{KD}}(x)}{\rho_{\text{FKD}}(x)} - 1\right| = 0, \quad \forall x\in D$$
\end{theorem}

Theorem \ref{theorem:2} implies that  $\rho_{\text{FKD}}$ uniformly approximates $\rho_{\text{KD}}$  when $n$ is sufficiently large. So it is safe for us to use it to replace  $\rho_{\text{KD}}$ in practice. This result is also verified by our numerical experiments in Section \ref{sec:experiment}.

\section{Experiments}
\label{sec:experiment}
In this section, we empirically evaluate the proposed kernel diffusion density functions against $\rho_{\text{naive}}$ and $\rho_{\text{LC}}$ in density-based clustering algorithms, and also compare them with other state-of-the-art methods. We denote $\rho_{\text{KD}}^{\text{sym}}$ and $\rho_{\text{FKD}}^{\text{sym}}$ as the kernel diffusion density functions and its fast surrogate, with symmetric-Gaussian kernel, respectively. Similarly, we denote $\rho_{\text{KD}}^{\text{asym}}$ and $\rho_{\text{FKD}}^{\text{asym}}$ as the proposed two density functions with asymmetric-Gaussian kernel, respectively. We examine their performance on a wide range of datasets. 
The clustering results are measured in Pairwise F-score~\citep{pairwise} and BCubed F-score~\citep{ bcubed}.

The parameters $\varepsilon$ (radius of the ball, used in $\rho_{\text{naive}}$, $\rho_{\text{LC}}$, $\rho_{\text{KD}}^{\text{sym}}$ and $\rho_{\text{FKD}}^{\text{sym}}$), $k$ (number of nearest neighbors, used in $\rho_{\text{LC}}$, $\rho_{\text{KD}}^{\text{asym}}$ and $\rho_{\text{FKD}}^{\text{asym}}$) are tuned by searching within a suitable range in the parameter space, and $h$ (bandwidth of Gaussian kernels, used in $\rho_{\text{KD}}^{\text{sym}}$, $\rho_{\text{FKD}}^{\text{sym}}$, $\rho_{\text{KD}}^{\text{asym}}$ and $\rho_{\text{FKD}}^{\text{asym}}$) is fixed to be $0.5$.

\begin{table}
 \caption{Clustering performance on benchmark datasets with different density functions applied to DPC. Pairwise F-score ($F_P$) and BCube F-score ($F_B$) under optimal parameter tuning are given. The best and second-bset results in each dataset are bolded and underlined, respectively.}
    \centering
    \setlength{\tabcolsep}{1.6mm}{
    \begin{tabular}{l|cccccc|cccccc}
    \toprule
    \multirow{2}{*}{Dataset} & \multicolumn{6}{c|}{$F_P$} & \multicolumn{6}{c}{$F_B$} \\
    \cmidrule{2-13}
    & $\rho_{\text{naive}}$ & $\rho_{\text{LC}}$ & \multicolumn{1}{|c}{$\rho_{\text{KD}}^{\text{sym}}$} & $\rho_{\text{KD}}^{\text{asym}}$ & $\rho_{\text{FKD}}^{\text{sym}}$ & $\rho_{\text{FKD}}^{\text{asym}}$ 
    & $\rho_{\text{naive}}$ & $\rho_{\text{LC}}$ & \multicolumn{1}{|c}{$\rho_{\text{KD}}^{\text{sym}}$} & $\rho_{\text{KD}}^{\text{asym}}$ & $\rho_{\text{FKD}}^{\text{sym}}$ & $\rho_{\text{FKD}}^{\text{asym}}$ \\
    \midrule
    Banknote & 
    54.3 & 31.6 & \multicolumn{1}{|c}{67.2} & \underline{83.9} & 67.2 & \textbf{93.6} & 
    57.7 & 31.8 & \multicolumn{1}{|c}{67.2} & \underline{85.1} & 67.2 & \textbf{93.6} \\
    Breast-d & 
    55.9 & 51.8 & \multicolumn{1}{|c}{\textbf{78.0}} & 69.1 & 67.4 & \underline{72.6} & 
    59.0 & 58.7 & \multicolumn{1}{|c}{\textbf{76.0}} & 69.7 & 69.4 & \underline{72.2} \\
    Breast-o & 
    57.6 & 70.7 & \multicolumn{1}{|c}{\underline{82.8}} & \textbf{92.9} & 82.7 & \textbf{92.9} & 
    52.2 & 74.1 & \multicolumn{1}{|c}{\underline{75.9}} & \textbf{92.2} & 75.8 & \textbf{92.2} \\
    Control & 
    48.6 & 49.3 & \multicolumn{1}{|c}{49.0} & \underline{63.9} & 52.5 & \textbf{64.5} & 
    51.6 & 52.4 & \multicolumn{1}{|c}{52.0} & \underline{70.8} & 55.1 & \textbf{71.8} \\
    Glass & 
    36.9 & 39.0 & \multicolumn{1}{|c}{46.3} & \textbf{48.1} & 44.8 & \underline{47.8} & 
    42.7 & 45.7 & \multicolumn{1}{|c}{55.1} & \underline{56.9} & 53.5 & \textbf{57.1} \\
    Haberman & 
    66.9 & 64.1 & \multicolumn{1}{|c}{74.5} & \underline{75.7} & \textbf{75.8} & \underline{75.7} & 
    66.9 & 63.3 & \multicolumn{1}{|c}{74.5} & \underline{75.8} & \textbf{75.9} & \underline{75.8} \\
    Ionosphere & 
    27.4 & 28.3 & \multicolumn{1}{|c}{46.9} & \textbf{54.9} & 46.0 & \underline{53.9} & 
    25.0 & 25.8 & \multicolumn{1}{|c}{42.6} & \textbf{52.5} & 41.7 & \underline{49.2} \\
    Iris & 
    54.3 & 53.8 & \multicolumn{1}{|c}{65.8} & \textbf{74.6} & \underline{69.2} & \textbf{74.6} & 
    61.6 & 62.3 & \multicolumn{1}{|c}{72.7} & \textbf{80.0} & \underline{74.0} & \textbf{80.0} \\
    Libras & 
    20.0 & 22.9 & \multicolumn{1}{|c}{29.3} & \textbf{31.5} & 26.0 & \underline{31.0} & 
    26.8 & 29.1 & \multicolumn{1}{|c}{37.8} & \textbf{41.8} & 33.3 & \underline{39.8} \\
    Pageblocks & 
    \underline{92.9} & \textbf{93.0} & \multicolumn{1}{|c}{90.5} & 90.2 & 89.7 & 90.2 & 
    \underline{89.9} & \textbf{90.0} & \multicolumn{1}{|c}{89.8} & 89.7 & 89.6 & 89.7 \\
    Seeds & 
    54.3 & 54.9 & \multicolumn{1}{|c}{68.0} & \textbf{78.0} & \underline{69.5} & \textbf{78.0} & 
    54.3 & 55.4 & \multicolumn{1}{|c}{72.4} & \textbf{78.7} & \underline{72.9} & \textbf{78.7} \\ 
    Segment & 
    48.4 & 48.0 & \multicolumn{1}{|c}{\underline{57.1}} & \textbf{58.0} & 41.4 & 56.1 & 
    64.2 & 63.8 & \multicolumn{1}{|c}{67.1} & \textbf{69.2} & 60.6 & \underline{68.2} \\
    Wine & 
    45.2 & 61.1 & \multicolumn{1}{|c}{56.6} & \textbf{68.0} & 60.0 & \underline{65.3} & 
    46.0 & 61.9 & \multicolumn{1}{|c}{61.5} & \textbf{74.7} & 66.3 & \underline{71.4} \\
    \bottomrule
    \end{tabular}}
    \label{tab:benchmark}
\end{table}

\begin{figure}[!b]
\caption{Precision-Recall curves of different approaches applied to DPC on MS1M dateset, using (a) Pairwise metric, and (b) BCubed metric .}
  \centering
  \includegraphics[width=1.0\textwidth]{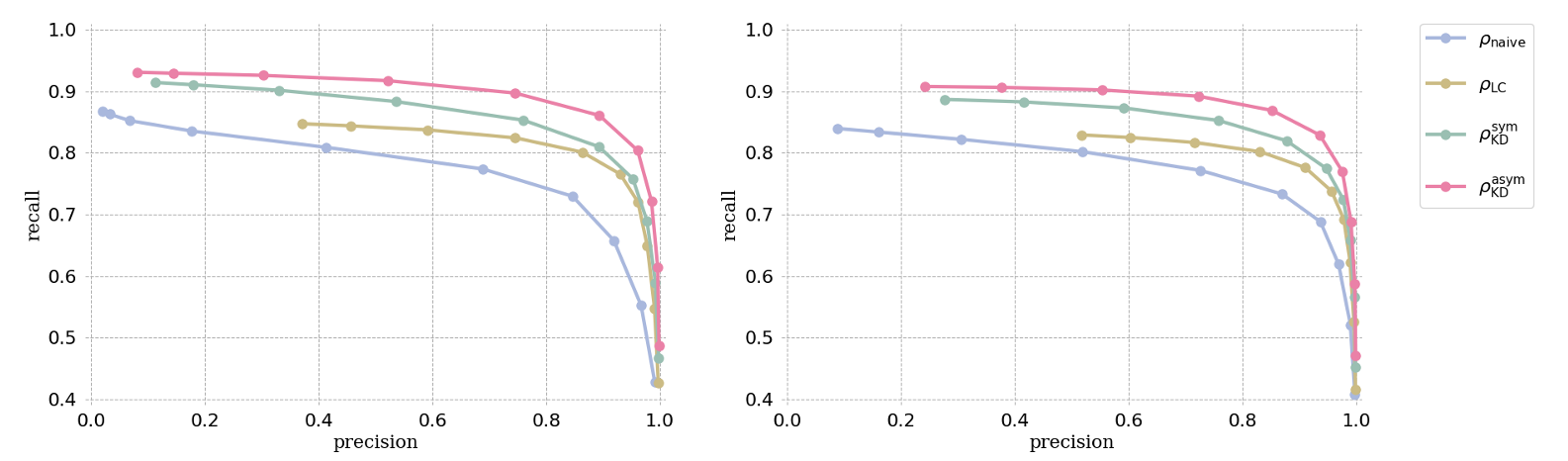}
  \label{fig2:pr}
\end{figure}

\subsection{Performance on Benchmark Datasets}
We now discuss the performance on 13 benchmark datasets (${\small\sim}100$ to ${\small\sim}5,000$ data points) from \cite{UCI} repository. The metadata is summarised in the Appendix. 

As summarised in Table~\ref{tab:benchmark}, both $\rho_{\text{KD}}^{\text{sym}}$ and $\rho_{\text{KD}}^{\text{asym}}$ uniformly outperform $\rho_{\text{naive}}$ and $\rho_{\text{LC}}$ in terms of clustering accuracy. The proposed kernel diffusion density functions with asymmetric Gaussian kernel, $\rho_{\text{KD}}^{\text{asym}}$, which enjoys better local adaptivity analytically, achieves the best results on most datasets and outperforms $\rho_{\text{naive}}$ and $\rho_{\text{LC}}$ by a large margin. 
It is worth noticing that the two fast surrogates, $\rho_{\text{FKD}}^{\text{sym}}$ and $\rho_{\text{FKD}}^{\text{asym}}$, achieve comparable results with their original counterparts, $\rho_{\text{KD}}^{\text{sym}}$ and $\rho_{\text{KD}}^{\text{asym}}$. In the Appendix similar results are observed for the same set of density functions applied to DBSCAN. 

\subsection{Performance on face image Datasets}
Clustering face images according to their latent identity becomes an important application in recent years. It is challenging in the sense that face image datasets usually contain thousands of identities, corresponding to thousands of clusters. Meanwhile, the number of images for each identity (cluster) is quite different, corresponding to the variety of cluster sizes. We assess the performance of the proposed approach on two popular face image datasets: emore\_200k~\citep{cdp} and MS1M~\citep{ms1m}.

\begin{wraptable}{r}{8.0cm}
\vspace{-18pt}
    \centering
    \caption{Clustering performance on emore\_200k. BCubed precision, recall and F-score are reported.}
    \setlength{\tabcolsep}{1.3mm}{
    \begin{tabular}{l|llll}
    \toprule
    & Algorithm & Precision & Recall & $F_B$ \\
    \midrule
    \multirow{4}{*}{Baseline} 
    & $k$-means & 94.24 & 74.89 & 83.45 \\
    & HAC & \textbf{97.74} & 88.02 & 92.62 \\
    & ARO & 52.96 & 16.93 & 25.66 \\
    & CDP & 89.35 & 88.98 & 89.16 \\
    \midrule
    \multirow{6}{*}{\shortstack{Density\\-based}}
    & \small{$\rho_{\text{naive}}$} & 92.36 & 78.14 & 84.65 \\
 
    & \small{$\rho_{\text{LC}}$} & 96.15 & 86.58 & 91.11 \\
    & \small{$\rho_{\text{KD}}^{\text{sym}}$} & 95.82 & 93.24 & 94.51 \\
    & \small{$\rho_{\text{KD}}^{\text{asym}}$} & 95.48 & \underline{93.82} & \underline{94.64} \\
    & \small{$\rho_{\text{FKD}}^{\text{sym}}$} & 95.27 & 92.54 & 93.89 \\
    & \small{$\rho_{\text{FKD}}^{\text{asym}}$} & \underline{96.37} & \textbf{93.93} & \textbf{95.13} \\
    \bottomrule
    \end{tabular}}
    \label{tab:emore}
\end{wraptable}

\textbf{emore\_200k.} The dataset contains 2,577 identities with 200,000 images following the protocol in~\cite{cdp}. Results are summarized in Table~\ref{tab:emore}. The proposed diffusion density functions are applied to DPC, and compared with $k$-means, HAC~\citep{hac}, ARO~\citep{aro}, and CDP~\citep{cdp}. Again, we observe significant improvement in the proposed density functions over $\rho_{\text{naive}}$ and $\rho_{\text{LC}}$. It is also worth pointing out that, density-based clustering with proposed kernel diffusion density functions also outperform the state-of-the-arts approaches such as CDP by a large margin.

\begin{wraptable}{r}{8.0cm}
\vspace{-18pt}
    \centering
    \caption{Clustering performance on MS1M. Pairwise F-score and BCubed F-score are reported.}
    \setlength{\tabcolsep}{1.3mm}{
    \begin{tabular}{l|lll}
    \toprule
    & Algorithm & $F_{P}$ & $F_{B}$ \\
    \midrule
    \multirow{4}{*}{Unsupervised} 
    & $k$-means  & 79.21  & 81.23  \\
    & HAC & 70.63 & 70.46 \\
    & ARO & 13.60 & 17.00 \\
    & CDP & 75.02 & 78.70 \\
    \midrule
    \multirow{3}{*}{Supervised} 
    & L-GCN & 78.68 & 84.37 \\
    & LTC & 85.66 & 85.52 \\
    & GCN(V+E) & \underline{87.55} & 85.94 \\
    \midrule
    \multirow{6}{*}{Density-based}
    & \small{$\rho_{\text{naive}}$} & 78.37 & 79.35 \\
    & \small{$\rho_{\text{LC}}$} & 83.61 & 85.06 \\
    & \small{$\rho_{\text{KD}}^{\text{sym}}$} & - & - \\
    & \small{$\rho_{\text{KD}}^{\text{asym}}$} & \textbf{88.15} & \underline{87.14} \\
    & \small{$\rho_{\text{FKD}}^{\text{sym}}$} & 84.40 & 85.37 \\
    & \small{$\rho_{\text{FKD}}^{\text{asym}}$} & 87.26 & \textbf{87.41} \\
    \bottomrule
    \end{tabular}}
    \begin{tablenotes}
      \footnotesize
      \item[1] - Not available due to limited computational power.
    \end{tablenotes}
    \label{tab:ms1m}
\end{wraptable}

\textbf{MS1M.} The dataset contains 8,573 identities with around 584,000 images following the protocols in~\cite{gcn-ve}. We reported the results of clustering  performance in Table~\ref{tab:ms1m}. Precision versus Recall curves for different density functions (applied to DPC) are plotted in Figure \ref{fig2:pr}. In Table~\ref{tab:ms1m}, the proposed kernel diffusion density functions outperform $\rho_{\text{naive}}$ and $\rho_{\text{LC}}$. Note that GCN-based methods such as L-GCN~\citep{lgcn}, LTC~\citep{ltc} and GCN (V+E)~\citep{gcn-ve} achieve generally better clustering performance than unsupervised methods  due to their supervised nature. However, it is quite encouraging to see that the proposed kernel diffusion approaches, although are also unsupervised clustering methods, considerably outperform the GCN-based methods.

\begin{figure}
  \caption{Sensitivity analysis on emore\_200k and MS1M. We investigate the clustering performance by varying the following parameters: (a) Radius of $\varepsilon$-ball; (b) Number $k$ of nearest neighbors; (c) Bandwidth $h$ of Gaussian kernel.}
  \centering
  \includegraphics[width=1.0\textwidth]{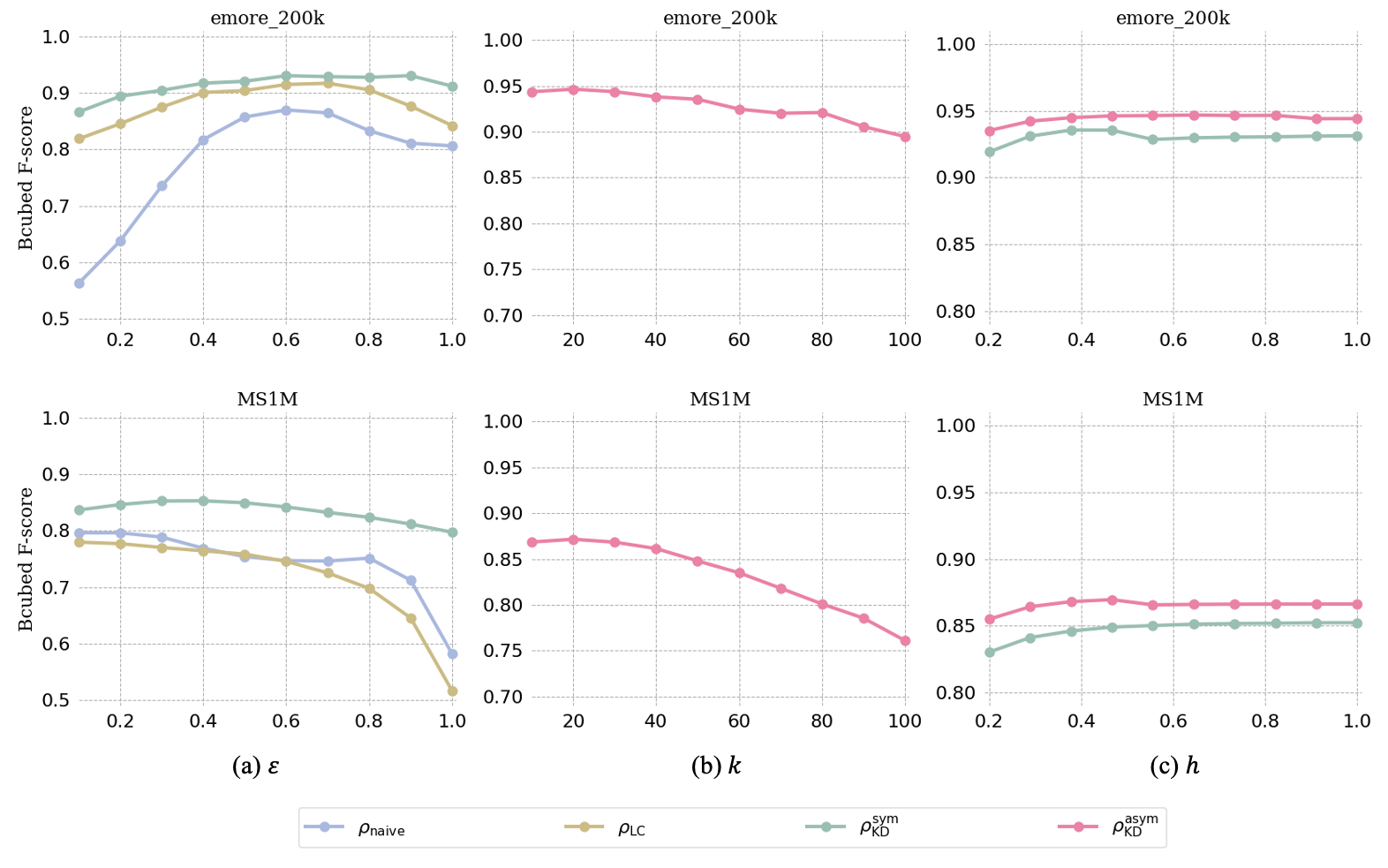}
   \label{fig:sensitivity}
\end{figure}

\subsection{Sensitivity Analysis}
Next, we examine the sensitivity of the proposed kernel diffusion density functions to hyper-parameters and compare it with $\rho_{\text{naive}}$ and $\rho_{\text{LC}}$. The results are obtained via extensive experiments on emore\_200k and MS1M, which are shown in Figure~\ref{fig:sensitivity}. We can see that the clustering performance of $\rho_{\text{KD}}^{\text{sym}}$ is much more stable than $\rho_{\text{naive}}$ and $\rho_{\text{LC}}$ when we vary the value of  $\varepsilon$. Whilst $\rho_{\text{KD}}^{\text{asym}}$ is robust to the parameter $k$, and both $\rho_{\text{KD}}^{\text{sym}}$ and $\rho_{\text{KD}}^{\text{asym}}$ are quite robust to  the parameter $h$. 

\begin{wrapfigure}{r}{7.8cm}
\vspace{-20pt}
 \caption{Running time and memory usage of the proposed methods at different sample sizes on MS1M.}
    \centering
    \includegraphics[width=0.5\textwidth]{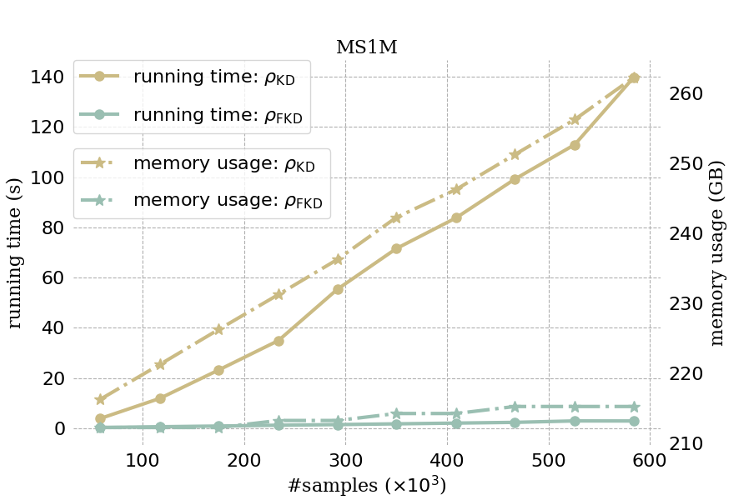}
\label{fig:time}
\end{wrapfigure}

\subsection{Computational Cost}

We carried out a series of experiments on MS1M to demonstrate the computational efficiency of the fast surrogate $\rho_{\text{FKD}}$ in terms of time and space. With a collection of subsampled data from MS1M at different percentile levels, we run both the kernel diffusion density $\rho_{\text{KD}}$ and the fast surrogate $\rho_{\text{FKD}}$. As we can observe from Figure~\ref{fig:time}, the running time and memory usage of $\rho_{\text{KD}}$ increase dramatically with the sample size. Whilst $\rho_{\text{FKD}}$ retains a very low level of computational cost. This suggests that $\rho_{\text{FKD}}$, which achieves an excellent computational efficiency, should be favored in practice.

\section{Conclusion}
Density-based clustering has a profound impact on machine learning and data mining. However, the underpinning naive density function suffers from detecting varying local features, causing extra errors in the clustering. We propose a new set of density functions based on the kernel diffusion process  to resolve this problem, which is adaptive to density regions of varying local distributional features. We demonstrate that DBSCAN and DPC adapted by the proposed approach have improved clustering performance comparing to their classic versions and other state-of-the-art methods.

\bibliography{iclr2022_conference}
\bibliographystyle{iclr2022_conference}

\appendix
\section{Appendix}
In this supplementary file, we provide technical proofs of the theoretical results in Section \ref{sec:fkd}, and present extra empirical experiments regarding our kernel diffusion approach with symmetric and asymmetric Gaussian kernels applied to DBSCAN. All the numerical experiments are carried out on a standard work station with a Intel 64-cores CPU and two Nvidia P100 GPUs.

\subsection{Proofs of Theoretical Result.}

\paragraph{Proof of Theorem \ref{theorem:3}.}
Since $\{D_{1},\cdots, D_{m}\}$ are disjoint, we have $p(x,y) = 0$ if $x$ and $y$ belong to different clusters. By the definition of matrix $P$, for each $x\in D_{j}$, we have
$$\int_{D}p(x,y)dF_{n}(y)=1,$$
which implies that
$$\int_{x\in D_{j}}\int_{y\in D_{j}}p(x,y)dF_{n}(x)dF_{n}(y)=|D_{j}|.$$
Therefore,
$$\bar{\rho}_{j}|D_{j}| = \int_{x\in D_{j}}\int_{y\in D_{j}}p(x,y)dF_{n}(x)dF_{n}(y)=|D_{j}|,$$
which implies that $\bar{\rho}_{j} = 1$ for any $j=1,,\dots,m$. \qed

Before proceeding to the proof of Theorem \ref{theorem:2}, we need following auxiliary lemma that relates the stationary distribution of a Markov chain to an arbitrary vector $g$. 




\begin{lemma}
\label{lemma:1}
Let $P$ be transition probability matrix of a finite inreducible discrete time Markov chain with $n$ states, which admits a stationary distribution, denoted by vector $\pi$. We write $e=(1,\dots,1)^T\in \mathbb{R}^n$ as the a column vector of ones. The following holds for any vector $g$ such that $g^Te\neq 0$:
\begin{itemize}
 \item[(1)] $(I-P+e g^T)$ is non-singular. 
 \item[(2)] Let $H=(I-P+e g^T)^{-1}$, then
${\pi}^T=g^TH.$
\end{itemize}

\end{lemma}
\begin{proof}
Since $\pi$ is the stationary distribution, we have $\pi^Te=1$. Applying Theorem 3.3 in \citep{hunter} yields that matrix $(I-P+e g^T)$ is non-singular. 

Next recall that $\pi^T P=\pi^T$, therefore we have
$$
\begin{aligned}
\pi^T(I-P+e g^T) & = \pi^T - \pi^TP + \pi^T eg^T\\
 & = \pi^T eg^T\\
 &= g^T,
\end{aligned}
$$
which implies $\pi^T=g^T H$.
\end{proof}





\paragraph{Proof of Theorem \ref{theorem:2}.}


Note that for for each $x\in D$, the linear reference function $\rho_{\text{FKD}}(x)=\int_D p(y,x)dF_n(y)$ is the corresponding column average of the transition matrix $P$.
We write the $i$-th column vector of $P$ as 
$$p_i=\big(p(x_1, x_i), \dots, p(x_n, x_i)\big)^T.$$ Therefore $\rho_{\text{FKD}}(x_i)=e^Tp_i/n$

Since the Markov chain induced by the kernel $k(x,y)$ is ergodic, the density $\rho(x, t)$ of the diffusion process $X_t$, will converge to the limiting stationary distribution of the Markov chain, denoted by $\pi$.

We can write the $n$-vectors of $g$ and $\pi$ in the following form: 
$$g=(g_1,\dots,g_n)^T=n\big(\rho_{\text{FKD}}(x_1), \dots, \rho_{\text{FKD}}(x_n)\big)^T \quad \mbox{and} \quad \pi=\big( \rho_{\text{KD}}(x_1),\dots, \rho_{\text{KD}}(x_n)\big)^T,$$
where $g_i=e^Tp_i$ is the $i$-th column sums of matrix $P$. As a result, we have 
$$\int_D \hat g(x) dF_n(x)=\dfrac{1}{n}e^T g=1, \quad \mbox{and}\quad n\int_D \hat \rho(x) dF_n(x)=e^T \pi=1.$$

By the definition of $g$, we know
$$
(eg^T)^2=n eg^T \quad \mbox{and} \quad e^TP= g^T . 
$$

It follows from Lemma \ref{lemma:1} that $(I-P+e g^T)$ is non-singular and ${\pi}^T=g^TH,$
where $H=(I-P+eg^T)^{-1}.$

We define $D=I+eg^T$. By simple algebra calculation, we can find $D$ is non-singular with 
$$
D^{-1}=I-\dfrac{eg^T}{n+1} .
$$

As a result, it is easy to see that that $g^TD^{-1}=\dfrac{g^T}{n+1}$ and 

$$H^{-1}=D-P=(I-PD^{-1})D.$$ 

Use the Neumann series, we have 
$$H=D^{-1}(I-PD^{-1})^{-1} = D^{-1}\sum_{i=0}^{\infty}(PD^{-1})^{i}.$$

Thus 
$$
\begin{aligned}
\pi^T-g^T/n=g^T\bigg(H-\dfrac{I}{n}\bigg)=g^T\bigg[D^{-1}\sum_{i=0}^{\infty}(PD^{-1})^{i}-\dfrac{I}{n}\bigg] 
\end{aligned}.
$$

Since we assume for any $x\in D$, $\hat g(x)<c \quad\mbox{for some } 0<c<1.$ This leads to
$$
g^Tp_j\le n c e^Tp_j= nc g_j.
$$ Therefore, let $\kappa_j$ be the $j$-th compoenent of $g^{T}PD^{-1}$, it is straightforward
$$\kappa_j \leq \frac{nc}{n+1}g_j\le cg_j.$$

This implies for every $x\in D$, 
$$
\begin{aligned}
|\rho_{\text{KD}}(x)-\rho_{\text{FKD}}(x)| &\leq {\rho_{\text{FKD}}(x)}\bigg|\dfrac{1}{n+1}\sum_{i=0}^{\infty}{c}^i-\dfrac{1}{n}\bigg|\\
 &\leq {\rho_{\text{FKD}}(x)} \bigg|\frac{1}{(n+1)(1-c)}-\dfrac{1}{n}\bigg|.
\end{aligned}
$$
Hence we have 
$\lim_{n\to \infty}|\frac{\rho_{\text{KD}}(x)}{\rho_{\text{FKD}}(x)} - 1| = 0$, which completes the proof. \qed

\subsection{Additional Experiment Results}
\paragraph{Metadata of benchmark datasets.}
The number of samples $n$, the number of clusters $c$, and feature dimension $d$ for each benchmark dataset are listed in Table~\ref{tab:meta} below.

\begin{table}[ht]
 \centering
 \caption{Metadata of benchmark datasets, includes sample size ($n$), the number of clusters ($c$), and feature dimension $d$.}
 \begin{tabular}{lccc}
 \toprule
 Dataset & $n$ & $c$ & $d$ \\
 \midrule
 Banknote & 1372 & 2 & 4 \\
 Breast-d & 569 & 2 & 30 \\
 Breast-o & 699 & 2 & 9 \\
 Control & 600 & 6 & 60 \\
 Glass & 214 & 7 & 9 \\
 Haberman & 306 & 2 & 3 \\
 Ionosphere & 351 & 2 & 34 \\
 Iris & 150 & 3 & 4 \\ 
 Libras & 360 & 15 & 90 \\
 Pageblocks & 5473 & 5 & 10 \\
 Seeds & 210 & 3 & 7 \\ 
 Segment & 210 & 7 & 19 \\
 Wine & 178 & 3 & 13 \\
 \bottomrule
 \end{tabular}
 \label{tab:meta}
\end{table}

\paragraph{Benckmark datasets with DBSCAN.} We provide the performance of the conventional density functions, $\rho_{\text{naive}}$ and $\rho_{\text{LC}}$, and the proposed kernel diffusion density functions with symmetric and asymmetric Gaussian kernels, $\rho_{\text{KD}}^{*}$ and $\rho_{\text{FKD}}^{*}$ ($* \in \{\text{sym}, \text{asym}\}$), applied to DBSCAN on 13 benchmark datasets. The results are summarised in Table~\ref{tab:benchmark_dbscan}. Similar to DPC, we see that both $\rho_{\text{KD}}^{\text{sym}}$ and $\rho_{\text{KD}}^{\text{asym}}$ uniformly outperform $\rho_{\text{naive}}$ and $\rho_{\text{LC}}$ in terms of clustering quality. $\rho_{\text{KD}}^{\text{asym}}$, which has better local adaptivity analytically, achieves the best results on most datasets and outperforms others by a significant margin in Breast-o, Control, Haberma and Seeds. 

\begin{table}[h]
    \caption{Clustering performance on benchmark datasets with different density functions applied to DBSCAN. Pairwise F-score ($F_P$) and BCube F-score ($F_B$) under optimal parameter tuning are given. The best and second-best results in each dataset are bolded and underlined, respectively.}
    \label{tab:benchmark_dbscan}
    \setlength{\tabcolsep}{1.6mm}{
    \begin{tabular}{l|cccccc|cccccc}
    \toprule
    \multirow{2}{*}{Dataset} & \multicolumn{6}{c|}{$F_P$} & \multicolumn{6}{c}{$F_B$} \\
    \cmidrule{2-13}
    & $\rho_{\text{naive}}$ & $\rho_{\text{LC}}$ & \multicolumn{1}{|c}{$\rho_{\text{KD}}^{\text{sym}}$} & $\rho_{\text{KD}}^{\text{asym}}$ & $\rho_{\text{FKD}}^{\text{sym}}$ & $\rho_{\text{FKD}}^{\text{asym}}$ 
    & $\rho_{\text{naive}}$ & $\rho_{\text{LC}}$ & \multicolumn{1}{|c}{$\rho_{\text{KD}}^{\text{sym}}$} & $\rho_{\text{KD}}^{\text{asym}}$ & $\rho_{\text{FKD}}^{\text{sym}}$ & $\rho_{\text{FKD}}^{\text{asym}}$ \\
    \midrule
    Banknote & 
    26.8 & 60.7 & \multicolumn{1}{|c}{62.0} & \textbf{66.4} & \underline{65.4} & \textbf{66.4} & 
    26.5 & 65.1 & \multicolumn{1}{|c}{60.7} & \textbf{67.4} & \underline{65.7} & \textbf{67.4} \\
    Breast-d & 
    56.7 & 63.0 & \multicolumn{1}{|c}{65.0} & \underline{66.6} & \textbf{67.2} & \underline{66.6} & 
    60.9 & 64.7 & \multicolumn{1}{|c}{66.0} & \underline{67.4} & \textbf{67.2} & \underline{67.4} \\
    Breast-o & 
    18.2 & 55.3 & \multicolumn{1}{|c}{59.2} & \textbf{70.6} & \underline{70.5} & \textbf{70.6} & 
    15.9 & 50.7 & \multicolumn{1}{|c}{52.3} & \textbf{71.3} & 70.6 & \underline{71.2} \\
    Control & 
    32.5 & 37.1 & \multicolumn{1}{|c}{51.0} & \textbf{60.3} & 48.9 & \underline{59.1} & 
    34.2 & 50.1 & \multicolumn{1}{|c}{53.7} & \textbf{66.9} & 51.6 & \underline{65.7} \\
    Glass & 
    22.0 & 29.8 & \multicolumn{1}{|c}{29.8} & \textbf{42.5} & \underline{42.0} & \textbf{42.5} & 
    25.8 & 36.9 & \multicolumn{1}{|c}{36.9} & \textbf{45.2} & \underline{43.5} & \textbf{45.2} \\
    Haberman & 
    68.6 & 72.2 & \multicolumn{1}{|c}{68.6} & \textbf{75.6} & 68.9 & \underline{75.3} & 
    69.2 & 73.1 & \multicolumn{1}{|c}{69.2} & \textbf{75.8} & 68.3 & \underline{75.7} \\
    Ionosphere & 
    25.9 & \underline{68.4} & \multicolumn{1}{|c}{68.0} & \textbf{74.2} & \textbf{74.2} & \textbf{74.2} & 
    23.8 & \underline{64.1} & \multicolumn{1}{|c}{63.7} & \textbf{72.1} & \textbf{72.1} & \textbf{72.1} \\
    Iris & 
    66.2 & 69.8 & \multicolumn{1}{|c}{66.2} & 57.2 & \textbf{73.7} & \underline{73.3} & 
    67.2 & 76.6 & \multicolumn{1}{|c}{67.2} & 67.0 & \textbf{79.4} & \underline{79.0} \\
    Libras & 
    \underline{18.1} & 12.0 & \multicolumn{1}{|c}{15.6} & 13.8 & \textbf{20.2} & 13.5 & 
    31.1 & 16.5 & \multicolumn{1}{|c}{\underline{42.1}} & \textbf{45.5} & 32.9 & 37.7 \\
    Pageblocks & 
    48.4 & 89.2 & \multicolumn{1}{|c}{\underline{90.0}} & \textbf{90.1} & 89.9 & \textbf{90.1} & 
    45.2 & 85.5 & \multicolumn{1}{|c}{\textbf{89.7}} & \underline{89.5} & \textbf{89.7} & \underline{89.5} \\
    Seeds & 
    57.8 & 47.6 & \multicolumn{1}{|c}{57.8} & \textbf{63.2} & 22.4 & \underline{62.4} & 
    59.2 & 53.0 & \multicolumn{1}{|c}{59.2} & \textbf{70.0} & 24.4 & \underline{69.2} \\ 
    Segment & 
    18.5 & \underline{47.9} & \multicolumn{1}{|c}{\textbf{54.8}} & 30.8 & 41.4 & 30.8 & 
    22.5 & 55.2 & \multicolumn{1}{|c}{\textbf{66.6}} & 53.6 & \underline{59.9} & 53.6 \\
    Wine & 
    40.5 & 40.5 & \multicolumn{1}{|c}{40.5} & \underline{49.5} & \textbf{50.0} & \underline{49.5} & 
    45.7 & 45.7 & \multicolumn{1}{|c}{45.7} & \textbf{52.3} & \underline{51.3} & \textbf{52.3} \\
    \bottomrule
    \end{tabular}}
\end{table}

\end{document}